\documentclass{article}

\usepackage{microtype}
\usepackage{graphicx}
\usepackage{subfigure}
\usepackage{booktabs} 

\usepackage[accepted]{icml2021}
\usepackage{times}
\usepackage{epsfig}
\usepackage{graphicx}
\usepackage{amsmath}
\usepackage{amssymb}
\usepackage{enumitem}

\usepackage{amsmath,amsfonts,bm}

\def\Tabref#1{Table~\ref{#1}}

\def\Figref#1{Figure~\ref{#1}}

\def\Secref#1{Section~\ref{#1}}

\def\Thmref#1{Theorem~\ref{#1}}
\def\Lemref#1{Lemma~\ref{#1}}

\def\Defref#1{Definition~\ref{#1}}
\def\Proref#1{Proposition~\ref{#1}}

\def\eqref#1{equation~\ref{#1}}

\def\1{\bm{1}}

\DeclareMathAlphabet{\mathsfit}{\encodingdefault}{\sfdefault}{m}{sl}
\SetMathAlphabet{\mathsfit}{bold}{\encodingdefault}{\sfdefault}{bx}{n}

\def\gD{{\mathcal{D}}}

\def\gL{{\mathcal{L}}}

\def\gU{{\mathcal{U}}}
\def\gV{{\mathcal{V}}}
\def\gW{{\mathcal{W}}}
\def\gX{{\mathcal{X}}}
\def\gY{{\mathcal{Y}}}

\def\sP{{\mathbb{P}}}
\def\sQ{{\mathbb{Q}}}

\newcommand{\E}{\mathbb{E}}

\newcommand{\R}{\mathbb{R}}

\usepackage{amsthm}
\theoremstyle{definition}
\newtheorem{definition}{Definition}[section]
\newtheorem{proposition}{Proposition}[section]
\newtheorem{theorem}{Theorem}[section]

\newtheorem{lemma}[theorem]{Lemma}

\usepackage[colorlinks=true,linkcolor=blue,citecolor=blue]{hyperref}%

\icmltitlerunning{Domain Adaptation with Factorizable Joint Shift}

\begin{document}

\twocolumn[
\icmltitle{Domain Adaptation with Factorizable Joint Shift}



\icmlsetsymbol{equal}{*}

\begin{icmlauthorlist}
\icmlauthor{Hao He}{mit}
\icmlauthor{Yuzhe Yang}{mit}
\icmlauthor{Hao Wang}{rutgers}
\end{icmlauthorlist}

\icmlaffiliation{mit}{Massachusetts Institute of Technology, CSAIL}
\icmlaffiliation{rutgers}{Department of Computer Science, Rutgers University}

\icmlcorrespondingauthor{Hao He}{haohe@mit.edu}

\icmlkeywords{Domain Adaptation}

\vskip 0.3in
]



\printAffiliationsAndNotice{}  

\begin{abstract}
Existing domain adaptation (DA) usually assumes the domain shift comes from either the covariates or the labels. However, in real-world applications, samples selected {from different} domains {could} have biases in both the covariates and the labels. In this paper, we propose a new assumption, Factorizable Joint Shift (FJS), to handle the co-existence of sampling bias in covariates and labels. Although allowing for the shift from both sides, FJS assumes the independence of {the bias between the two factors}. We provide theoretical and empirical understandings about when FJS degenerates to prior assumptions and when it is necessary. We further propose Joint Importance Aligning (JIA), a discriminative learning objective to obtain joint importance estimators for both supervised and unsupervised domain adaptation. 
Our method can be seamlessly incorporated with existing domain adaptation algorithms for better importance estimation and weighting on the training data.
Experiments on a synthetic dataset demonstrate the advantage of our method.
\end{abstract}

\vspace{-5mm}
\section{Introduction}

The core problem in domain adaptation originates from the \emph{shift} between the training distribution (source domain) $p$ and the test distribution (target domain) $q$. This equivalently indicates that the samples we obtain from one domain are biased comparing to that from the other domain. Depending on different beliefs about the cause of the sampling bias, there emerges two common domain adaptation assumptions: (1) \emph{Covariate shift (CS)} assumes the bias is purely on the covariate while $p(y|x)$ and $q(y|x)$ equals. (2) \emph{Label shift (LS)} assumes the bias is purely on the label while $p(x|y)$ and $q(x|y)$ equals. As observed by~\cite{scholkopf2012causal}, covariate shift corresponds to causal learning and label shift corresponds to anti-causal learning. 

Another widely adopted assumption is \emph{domain invariance (DI)}, which assumes the discrepancy in two domains can be eliminated by feature transformations. One typical example is transferring from MNIST to colorful-MNIST, where a transformation discarding the color information will force the feature distribution invariant to the domains. Such assumption motivates many domain adaptation methods based on adversarial training~\cite{DANN,CIDA,CDAN}. A recent work~\cite{GLS} proposed a new assumption called \emph{generalized label shift (GLS)}, which aims to generalize domain adversarial techniques to handle mismatched label distributions.

\begin{figure}[!t]
\begin{center}
\centerline{\includegraphics[width=0.99\columnwidth]{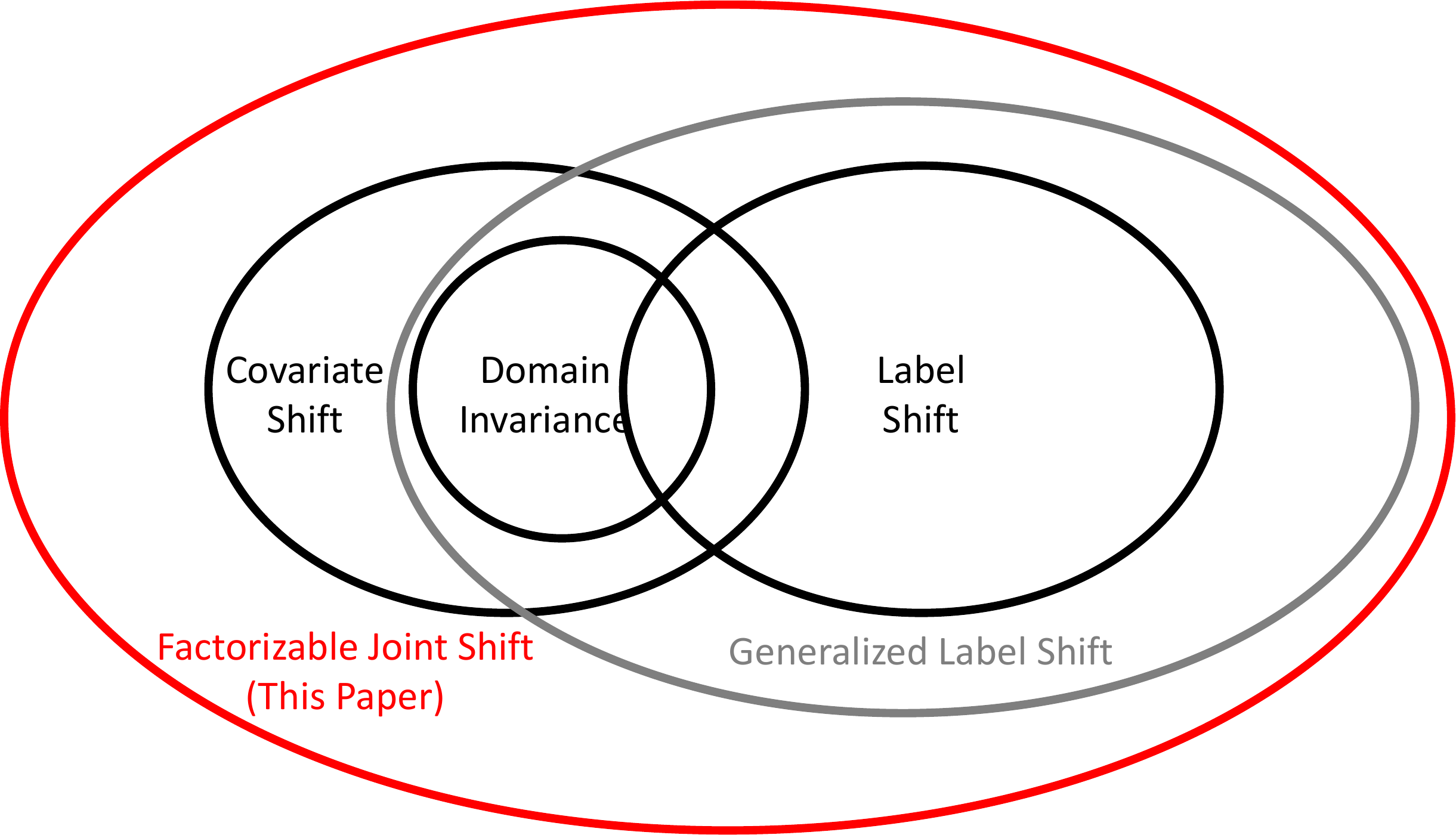}}
\caption{The Venn Diagram of domain adaptation assumptions.}
\label{fig:venn}
\end{center}
\vspace{-1.2cm}
\end{figure}

In this paper, we focus on a more relaxed setting, called \emph{Factorizable Joint Shift (FJS)}, where we do not assume any domain invariance and allow the bias to exist on both the data and the label. We only assume the biases coming from the data and the label are statistically independent. Through the lens of joint importance, we unify the traditional assumptions (see \Figref{fig:venn}) and identify that the uncertainty of the learning task affects the choice of assumptions.

We further propose a general strategy \emph{Joint Importance Aligning (JIA)} to handle FJS. Essentially, JIA is a discriminative learning paradigm that learns the two factors of joint importance corresponding to the data and the label. We also identify, under the setting of unsupervised domain adaptation, there is a lack of information to discover the factorization of joint importance. Motivated by the smoothness of the importance, we propose to add a regularization to facilitate the clustering structure of the estimated importance. We conduct experiments on a synthetic dataset where our method outperforms all the baselines.

\begin{table*}[!t]
\label{tab:factor}
\caption{Formal definition and importance factorization of traditional domain adaptation assumptions.}
\vspace{-2mm}
\begin{center}
\resizebox{1\textwidth}{!}{
\begin{tabular}{lcc}
\hline
Assumptions & Mathematical Definitions & Joint Importance Factorizations \\
\hline
\hline
Covarite Shift & 
$\gD_S(X)\neq\gD_T(X)$, $\gD_S(Y|X)=\gD_T(Y|X)$ & 
$\gU(x) = \frac{\gD_T(x)}{\gD_S(x)}$, $\gV(y) = 1$
\\[3pt]
Label Shift & 
$\gD_S(Y)\neq\gD_T(Y)$, $\gD_S(X|Y)=\gD_T(X|Y)$ & 
$\gU(x) = 1$, $\gV(y) = \frac{\gD_T(y)}{\gD_S(y)}$
\\[3pt]
Domain Invariance & 
\small{$\exists Z=g(X)$, s.t. $\gD(Y|g(X))=\gD(Y|X)$, $\gD_S(Z,Y)=\gD_T(Z,Y)$} & 
$\gU(x) = \frac{\gD_T(x)}{\gD_S(x)}$, $\gV(y) = 1$
\\[3pt]
\small{Generalized Label Shift} & \small{$\exists Z=g(X)$, s.t. $\gD(Y|g(X))=\gD(Y|X)$, $\gD_S(Z|Y)=\gD_T(Z|Y)$} & 
\small{$\gU(x)=\frac{\gD_T(X=x|Z=g(x))}{\gD_S(X=x|Z=g(x)) },\gV(y)=\frac{\gD_T(y)}{\gD_S(y)}$}
\\[3pt]
\hline
\end{tabular}
}
\end{center}
\vspace{-6mm}
\end{table*}

\section{Assumption: Factorizable Joint Shift}

\noindent {\bf Notation.}
We use $\gX$ and $\gY$ to denote the data space and the label space; $x$ and $y$ denote values in the two spaces; $X$ and $Y$ are random variables taking value in the two spaces. $\gD_S$ and $\gD_T$ denote the sample distribution in the source and the target domains. When the statement holds in both domains, we use $\gD$ to denote the distribution, ignoring the domain subscript. We use $\gD(V_1)$, $\gD(V_1|V_2)$, and $\gD(V_1,V_2)$ to denote the marginal, conditional, and joint distribution of certain variables in the domain, respectively. We also use $\gD(x)$ as an abbreviation of $\gD(X=x)$ which is the probability (density) of a variable taking a certain value. 
Some assumptions involve the feature transformation of the data. By default, we denote $g$ as the feature encoder and $Z \triangleq g(X)$ as the induced random variable of the feature.

\begin{definition}[Joint Importance Weight]\label{def:jiw}
Joint importance weight $\gW$ is the ratio between joint distribution in the target domain $T$ and that in the source domain $S$, i.e., 
\begin{equation*}
\gW(x,y) \triangleq \frac{\gD_T(X=x,Y=y)}{\gD_S(X=x,Y=y)},~~~\forall x,y \in \gX \times \gY.
\end{equation*}
\end{definition}

\begin{definition}[Factorizable Joint Shift (FJS)]\label{def:fjs}
Two domains $\gD_S$ and $\gD_T$ satisfy the \emph{Factorizable Joint Shift} if their joint importance weight $\gW(x,y)$ can be factorized, i.e., $\gW(x,y) = \gU(x) \gV(y)$, where $\gU$ and $\gV$ are the data and the label importance factors, respectively.
\end{definition}

We first mathematically define Factorizable Joint Shift (FJS) in \Defref{def:fjs}. Note that to avoid the discussion on invalid importance weights, throughout this paper, we assume the source domain has a larger support than the target domain, i.e., $\gD_T(x,y) > 0 \Rightarrow \gD_S(x,y) > 0$. As illustrated in \Figref{fig:venn}, FJS is a \emph{relaxed} domain adaptation assumption and contains all four previously mentioned assumptions as subsets. In \Tabref{tab:factor}, we provide mathematical definitions of each assumption as well as its importance factorization. The factorizations for covariate shift and label shift are obvious by their definitions. Domain invariance exhibits the exact factorization as covariate shift since it is a subset of covariate shift. Yet, the factorization of generalized label shift requires more efforts. Due to the space limit, the proof is deferred to Appendix~\ref{appendix:proof}.

In the rest of the section, we delve deeper into FJS. We study the scope of FJS from two perspectives: 1) Are there any principles about when FJS is required? 2) What are the cases that can only be handled by FJS while beyond the scope of previous assumptions in \Tabref{tab:factor}?


\subsection{When FJS is not required: an Analysis}

When we study the scope of FJS, especially comparing it with GLS, which is a prior assumption most similar to FJS, we observe that the key difference between FJS and GLS is also the key difference between CS and DI. Specifically, FJS and CS generally assume there exists shift in the data marginal distributions, while GLS and DI further constrain such data distribution shift by assuming the shift can be eliminated via a proper feature encoder.
However, in the current stage of research on domain adaption, domain adversarial training techniques tailored for DI are widely adopted, resulting in the ambiguity on the boundary between DI and CS for practical problems. Thinking along this line, we observe an interesting yet common phenomenon that, for most of the learning tasks used in domain adaptation literature, DI and CS make no differences. We further identify the reason behind is that the chosen tasks all target on learning \emph{deterministic} functions. For example, consider the most widely used computer vision tasks such as digit or object classification. Given an image, the corresponding output is quite certain. As stated in \Thmref{thm:det_cs_di} and \Thmref{thm:det_fjs_gls}, we theoretically prove that such determinacy between data and label degenerates CS and FJS, reducing them to the same as DI and GLS. Proofs can be found in Appendix~\ref{appendix:proof}.

\begin{theorem}[Determinacy + Matched Label + CS $\Rightarrow$ DI]\label{thm:det_cs_di}
If there exists a deterministic mapping from data to label, and the two domains share the same label distribution, then the source and target domain satisfying covariate shift assumption implies they satisfy domain invariance.
\end{theorem}
\vspace{-1mm}

\begin{theorem}[Determinacy + FJS $\Rightarrow$ GLS]\label{thm:det_fjs_gls}
If there exists a deterministic mapping from data to label, then the source and target domain satisfying factorizable joint shift assumption implies they satisfy generalized label shift.
\end{theorem}
\vspace{-1mm}

Our theoretical findings are meaningful in the two fronts: On the one hand, we prove that using domain invariance and generalized label shift assumptions are quite ``safe'' when the task of interest is deterministic; On the other hand, we question the validity of assuming the existence of domain invariant representations when dealing with tasks involving clear uncertainty, such as age estimation from images.

\subsection{When you need FJS: a Motivating Example}
\label{sec:toy-example}

\begin{figure}[!ht]
\begin{center}
\centerline{\includegraphics[width=1.0\columnwidth]{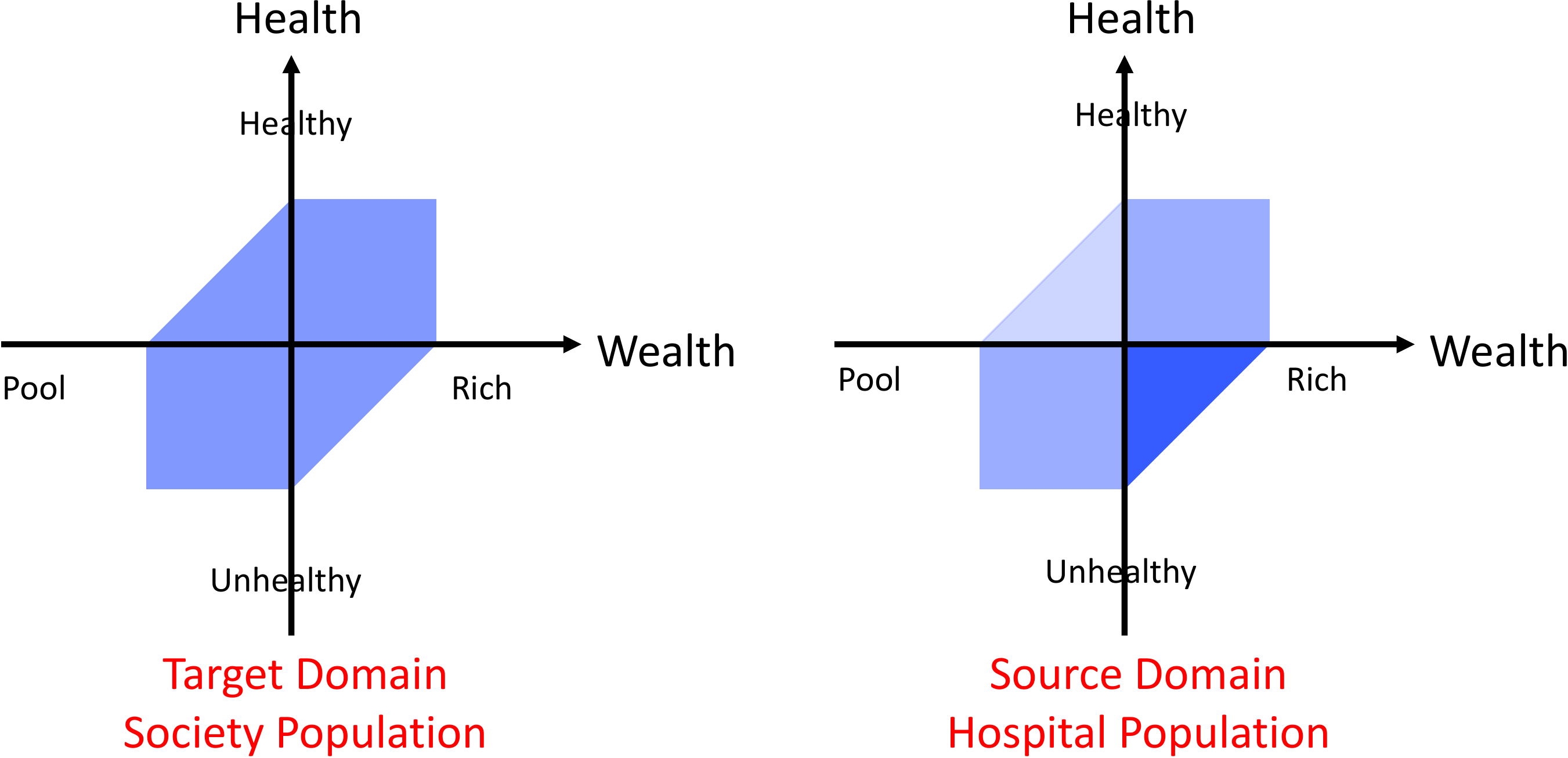}}
\vspace{-3mm}
\caption{A toy example that satisfies FJS. The plots show the joint distributions of wealth and health. A deeper color indicates a larger probability density. \textbf{Left}: the joint distribution on the whole society which assumed to be uniform in the hexagon. \textbf{Right}: the joint distribution observed in a hospital which exhibits biases.}
\label{fig:toy}
\end{center}
\end{figure}
\vspace{-8mm}

In this section, we construct a toy example which is within the assumption of FJS while outside of all assumptions in \Tabref{tab:factor}. The task is to predict one's health condition $y$ from her wealth $x$. With normalization, we abstract one's health and wealth as real values from $[-1,1]$. As shown in \Figref{fig:toy}, we assume our target domain is the whole society population with a uniformly distribution over a hexagon on the wealth-health 2D plane. Further, the source domain is the population that observed in a hospital. Note that the source domain distribution is no longer uniform due to the bias from both data and label. Intuitively, we create biases for two reasons: (1) Along the ``health'' axis, populations from hospitals intrinsically induce bias towards unhealthy course. (2) Along the ``wealth'' axis, sampling from the hospital population also causes an implicit shift on the wealth distribution, where a preference is imposed on the rich direction.
As a result, the source domain has a skewed distribution as shown in the right side of \Figref{fig:toy}. In this toy example, the joint importance has the following factorization, $\frac{\gD_T(x,y)}{\gD_S(x,y)} = \gU(x)\gV(y)$, where $\gU(x) = u_{\rm{rich}} \cdot \1[x \geq 0] + u_{\rm{poor}} \cdot\1[x<0]$ and $\gV(x) = v_{\rm{healthy}} \cdot \1[y \geq 0] + v_{\rm{unhealthy}} \cdot\1[y<0]$. It clearly satisfies FJS and does not satisfy CS and LS. In the Appendix, we prove this toy example is also beyond GLS.

\vspace{-2mm}
\section{Method: Joint Importance Aligning}
We propose Joint Importance Aligning (JIA), which performs deep domain adaptation with factorizable domain shift. JIA employs two deep model, $U(x;\theta_u)$ and $V(y;\theta_v)$, to learn the data and label importance factor simultaneously. Once the importance is learned, we can apply any existing domain adaptation algorithms with learned weights on the source domain data during training.

In this paper, we extend the well-known Domain Adversarial Neural Networks~(DANN)~\cite{DANN} with JIA. We refer the resulting pipeline to as Joint Importance Aligning Domain Adaptation~(JIADA). It has three players: encoder $E$, predictor $F$ and discriminator $D$. With learnt importance factors $U,V$, the three players perform the following optimization to learn to align the feature and make the prediction:
\vspace{-2mm}
{
\begin{align*} \label{eq:jiada}
\max_{E}\min_{D}~~~ &\E_{x,y\sim \gD_S} {\color{red}U(x)V(y)}\log(D(E(x))) \\
&~~~~~~~+\E_{x \sim \gD_T} \log(1-D(E(x))), \\
\min_{F,E}~~~ &\E_{x,y\sim \gD_S} {\color{red}U(x)V(y)} \gL_{\rm{pred}}(y,F(E(x))),
\end{align*}
}
\vspace{-1mm}
where $\gL_{\rm{pred}}$ is the prediction loss (e.g., cross-entropy for classification or $L_2$ for regression). As indicated by the red color, $U(x)$ and $V(y)$ are the newly added terms comparing to the original DANN.


To estimate the importance factor, we propose two objectives for the supervised and unsupervised DA settings:
\vspace{-2mm}
{\small
\begin{align*} 
\gL_{\rm{sup}}(U,V) = &\E_{x, y \sim \gD_S} \log(1 + U(x)V(y)) \\
&~~~~+\E_{x, y \sim \gD_T} \log(1 + \frac{1}{U(x)V(y)}),\\
\gL_{\rm{unsup}}(U,V) = &\E_{x \sim \gD_S}\log(1 + U(x)\widetilde{V}(x)) \\
&~~~~+\E_{x \sim \gD_T} \log(1 + \frac{1}{U(x)\widetilde{V}(x)}).
\end{align*}
}
\vspace{-5mm}

Note that in unsupervised objective, we define $\widetilde{V}(x) \triangleq \E_{y \sim \gD_S(y|x)} V(y)$ which is the label importance factor's effect on the data. 

\begin{proposition}\label{prop:sup}
The importance factors $U,V$ achieve the minima of the objective $\gL_{\mbox{sup}}$ if and only if the their product equals the true joint importance. Formally, $U,V \in \arg\min_{U',V'} \gL_{\mbox{sup}}(U',V')$ is equivalent to $U(x)V(y) \equiv \gU(x)\gV(y)$.
\end{proposition}

\begin{proposition}\label{prop:unsup}The importance factors $U,V$ achieve the minima of the objective $\gL_{\mbox{unsup}}$ if and only if they matches the source data marginal distributions to the target domain. Formally, $U,V \in \arg\min_{U',V'} \gL_{\mbox{sup}}(U',V')$ is equivalent to, for any $x$, $\int_{y}\gD_S(x,y)U(x)V(y)dy = \gD_T(x)$.
\end{proposition}

In Proposition \ref{prop:sup} and \ref{prop:unsup}, we theoretically show the correctness of these objectives. However there is still one difficulty remains in the unsupervised setting. Matching the data distribution of source and target domains does not guarantee to find the true importance factorization. This is because learning could collapse to a trivial solution, $U(x)=\frac{\gD_T^X(x)}{\gD_S^X(x)}$ and $V(y) = 1$, which ignores the label importance factor. To avoid the data importance factor overfitting at each point, we force it to have limited number of different values. Specifically, we ask the data importance factor to divide the data space into $K$ sub-domains and assign a constant importance weight for each sub-domain. Mathematically, to achieve this, we parameterize the data importance factor by the product of a $K$-way classifier $C(x;\theta_c)$ and a score vector $s \in \R^{K}$, i.e., $U(x) = \sum_{k=1}^K s_k C_k(x)$ where $C$ outputs a softmax result and $C_k$ is the probability of data $x$ being in the $K^{\text{th}}$ sub-domain.


\begin{figure}[!t]
\begin{center}
\centerline{\includegraphics[width=1.0\columnwidth]{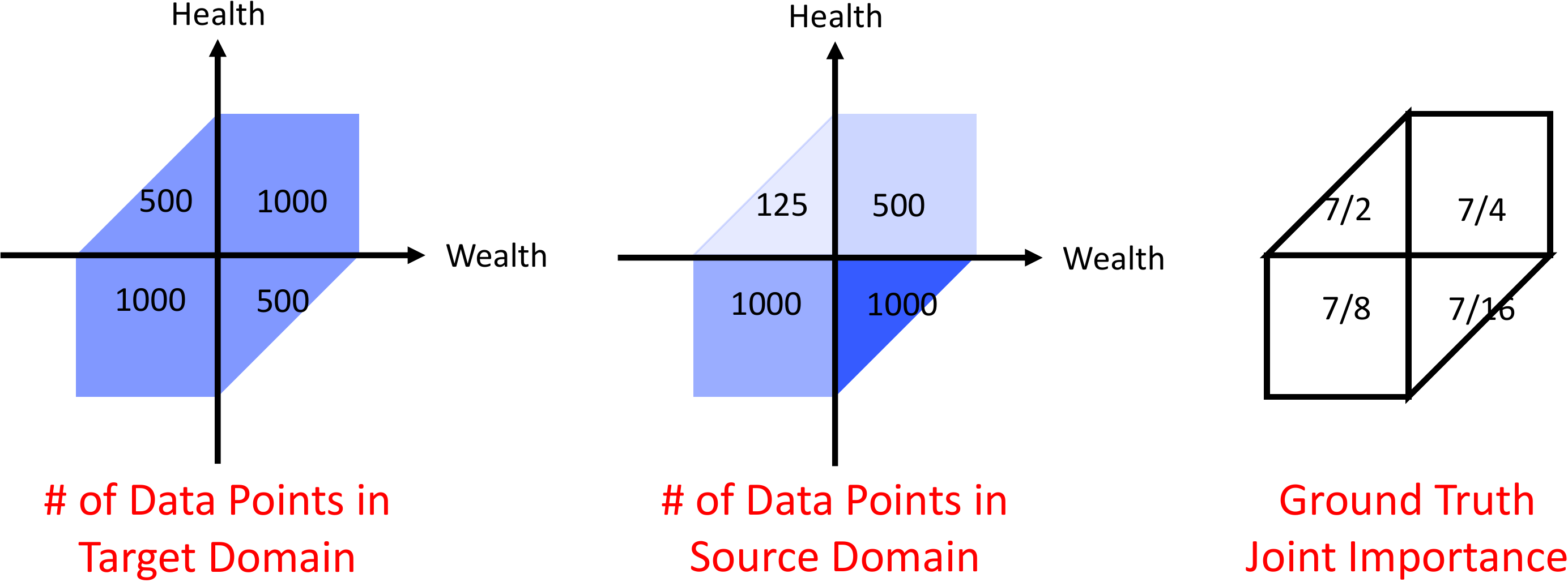}}
\vspace{-3mm}
\caption{Statistics of the synthetic toy dataset.}
\label{fig:dataset}
\end{center}
\vspace{-8mm}
\end{figure}

\begin{table}[!t]
\caption{Quantitative results for unsupervised domain adaptation on the toy dataset. Best performance is shown in bold.}
\label{tab:res-unsup}
\vspace{-2mm}
\begin{center}
\resizebox{0.48\textwidth}{!}{
\begin{tabular}{lcc}
\hline
Assumptions & Methods & NLL~($\downarrow$) \\
\hline\hline
 & Target Only & 0.600 $\pm$ 0.002 \\
\hline\hline
\small{No Shift} & Source Only & 0.765 $\pm$ 0.021 \\
\small{Covarite Shift} & SSBC~\cite{SSBC} & 0.771 $\pm$ 0.016 \\
\small{Label Shift} & BBSC~\cite{BBSC} & 0.773 $\pm$ 0.044 \\
\small{Domain Invariance} & DANN~\cite{DANN} & 0.772 $\pm$ 0.021 \\
\small{Generalized Label Shift} & IWDAN~\cite{GLS} & 0.772 $\pm$ 0.030 \\
\small{Factorizable Joint Importance} & JIADA (Ours) & \textbf{0.626 $\pm$ 0.011} \\
\hline
\end{tabular}
}
\end{center}
\vspace{-0.6cm}
\end{table}

\section{Experiments}
\label{sec:exp}

\noindent
\textbf{Dataset.} We construct the dataset following our toy example in \Secref{sec:toy-example}. As shown in \Figref{fig:dataset}, we sample 3000 data points uniformly from the hexagon and treat them as target domain data. For the source domain data, there are 125, 500, 1000, and 1000 points generated in the four quadrants, respectively. The joint importance is different in each quadrant while being factorizable.
\medskip\\
\textbf{Task and Evaluation Metric.} Our task is to predict one's health $y$ using her wealth $x$. For each input $x$, we ask the model to predict the mean $\mu(x)$ and variance $\sigma(x)$ of the label. To evaluate the model, we employ the target domain negative log-likelihood~(NLL) as the criterion (lower the NLL, better the model learns $\gD_T(Y|X)$):
{\small
\begin{equation}
\vspace{-0.2cm}
\mbox{NLL}(\mu,\sigma) = \E_{x,y\sim \gD_T} \left[ \frac{(y - \mu(x))^2}{2\sigma(x)^2} +\log(\sqrt{2\pi} \sigma(x)) \right].
\vspace{-0.15cm}
\end{equation}
}
\\
\textbf{Baselines.} We use the following baselines in the unsupervised DA setting: \emph{Sample Select Bias Correction~(SSBC)} \cite{SSBC}, \emph{Black Box Shift Correction (BBSC)} \cite{BBSC}, \emph{Domain Adversarial Neural Networks (DANN)} \cite{DANN}, and \emph{Importance Weighting Domain Adversarial Networks (IWDAN)} \cite{GLS}. As listed in \Tabref{tab:res-unsup}, each of them corresponds to one aforementioned DA assumption. We also include \emph{Source Only} and \emph{Target Only}, i.e., models that simply trained in the source or target domain. More details are in Appendix~\ref{app:baseline}.
\medskip

\begin{figure}[!t]
\begin{center}
\centerline{\includegraphics[width=0.9\columnwidth]{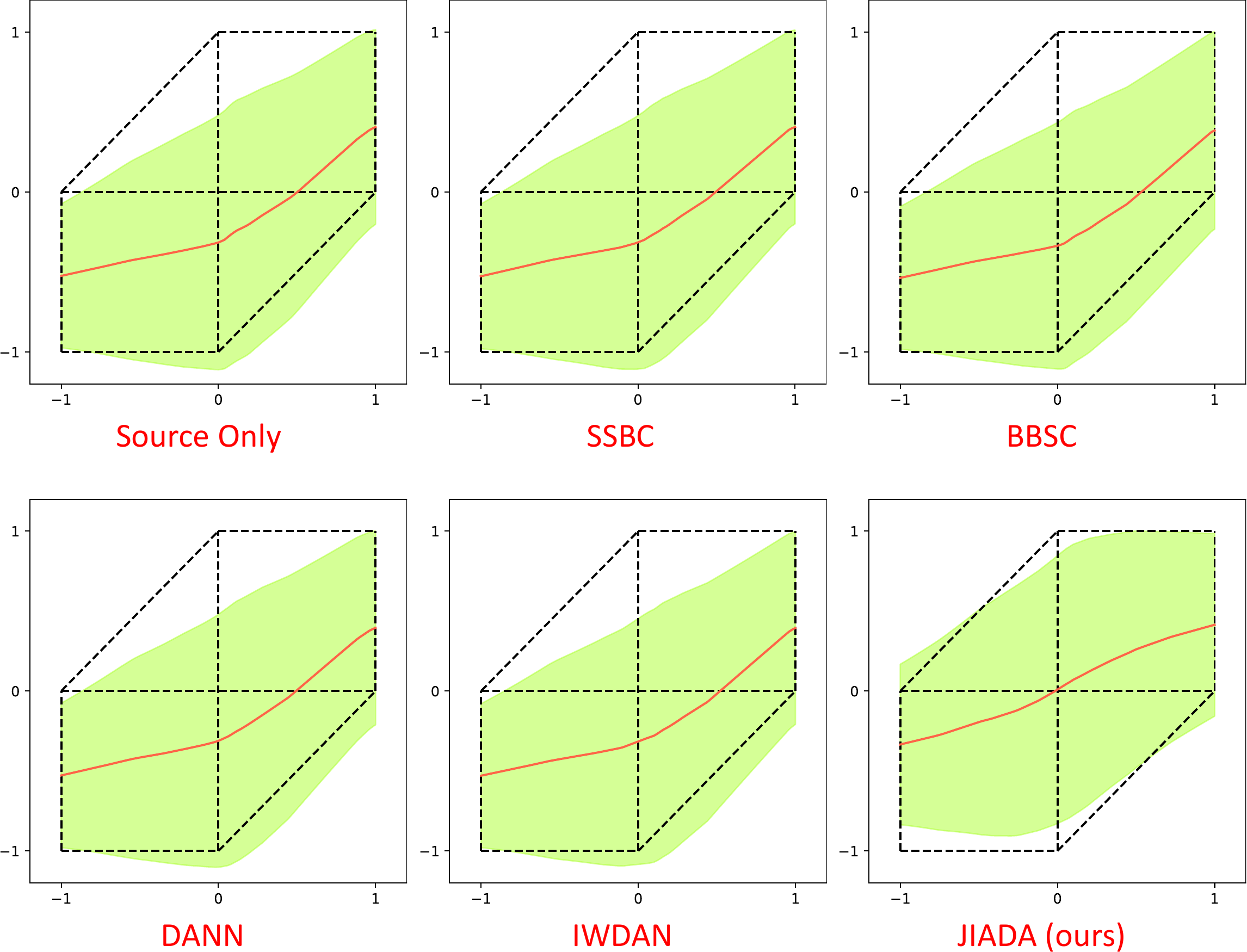}}
\caption{Qualitative results for unsupervised domain adaptation on the toy dataset. Red curves represent the predicted mean $\mu(x)$, green regions show the area deviated from the mean by $\pm \sqrt{3}\sigma(x)$.}
\label{fig:result}
\end{center}
\vspace{-0.5cm}
\end{figure}

\begin{figure}[!t]
\begin{center}
\centerline{\includegraphics[width=1.0\columnwidth]{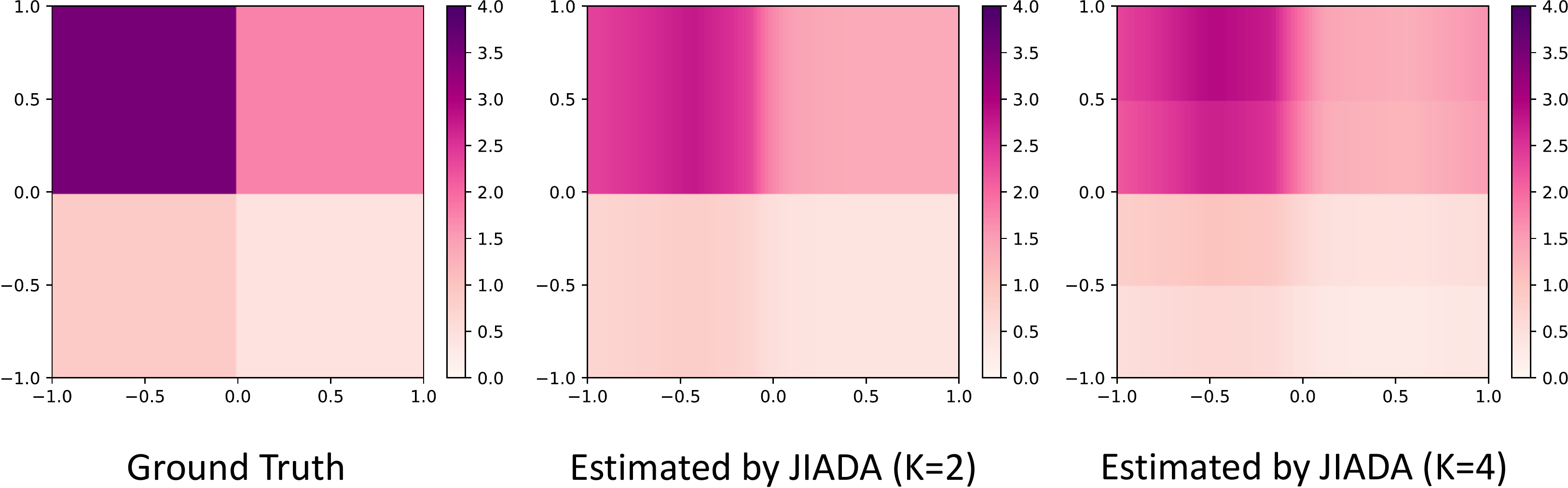}}
\caption{Visualization of the joint importance. We show ground truth as well as estimations by JIADA with different $K$.}
\label{fig:vis-iw}
\end{center}
\vspace{-1cm}
\end{figure}

\noindent \textbf{Quantitative Results.} We first show quantitative results in Table~\ref{tab:res-unsup}. As the table reports, compared to all baseline methods, JIADA significantly improves the results, achieving the best NLL. We also note that other methods do not show advantages when compared to \emph{Source Only} model. This sheds light on the limitations of current DA assumptions when uncertainty comes into the picture. Finally, the performance of JIADA is also comparable to \emph{Target Only} model, highlighting the effectiveness of the joint importance aligning.
\medskip\\
\textbf{Qualitative Results.} Furthermore, we show qualitative results in \Figref{fig:result}, where we plot the estimated mean and variance for each method. 
We note that a random variable uniformly distributed in $[a,b]$ has the first and second moments as $\frac{a+b}{2}$ and $\frac{(b-a)^2}{12}$. If we learn this variable via Gaussian parameterization, the optimal Gaussian parameters should be $\mu=\frac{a+b}{2}$ and $\sigma=\frac{b-a}{2\sqrt{3}}$. Thus comparing $[a,b]$ with $[\mu-\sqrt{3}\sigma,\mu+\sqrt{3}\sigma]$ reflects how well the Gaussian aligns with that variable.
The figure shows JIADA achieves the best estimates among all methods, where the predictions are mostly aligned with the ground truth.
\medskip\\
\textbf{Estimated Importance.} We visualize the joint importance and our method's estimation results in \Figref{fig:vis-iw}. We observe that JIADA faithfully infers the underlying joint importance with small errors, and is robust to different hyper-parameters. Interestingly, JIADA automatically identifies the group structure and treats $x\geq 0$ and $x<0$ differently.
\medskip\\
\textbf{Closing Remarks.} We have introduced a relaxed domain adaption assumption through a joint importance perspective. The proposed FJS and JIADA are more generic and practical, extending the scope of current DA methods. We believe the next step is to test it in larger datasets and real-world tasks across different applications.

\newpage


\bibliography{paper}
\bibliographystyle{icml2021}

\appendix

\section{Proofs}\label{appendix:proof}

\begin{proof}[Proof of Domain Invariance $\Rightarrow$ Covariate Shift]
By domain invariance, we have $\gD(Y|X=x) = \gD(Y|Z=g(x))$ holds in both domains. We also knonw the joint distribution of feature and label is the same in the two domian, i.e., $\gD_S(Y,Z) = \gD_T(Y,Z)$ which gives the equality of condition distribution
$\gD_S(Y|Z) = \gD_T(Y|Z)$.
As a result, $\gD_S(Y|X=x) = \gD_S(Y|Z=g(x)) = \gD_T(Y|Z=g(x)) = \gD_T(Y|X=x)$ which means the covariate shift holds.
\end{proof}
 
\begin{proof}[Proof of Importance Factorization of GLS]
Under the GLS assumption, the feature $Z=g(X)$ has properties: $\gD_S(Y|X)=\gD_S(Y|Z)$, $\gD_T(Y|X)=\gD_T(Y|Z)$ and $\gD_S(Z|Y)=\gD_T(Z|Y)$. Hence,
\begin{align*}
\frac{\gD_T(X,Y)}{\gD_S(X,Y)} = \frac{\gD_T(X) \gD_T(Y|X)}{\gD_S(X)\gD_S(Y|X)} = \frac{\gD_T(X) \gD_T(Y|Z)}{\gD_S(X) \gD_S(Y|Z)} \\
= \frac{\gD_T(X) \gD_T(Y,Z) /\gD_T(Z)  }{\gD_S(X) \gD_S(Y,Z) / \gD_S(Z)}
= \frac{\gD_T(X|Z) \gD_T(Y)}{\gD_S(X|Z) \gD_S(Y)}
\end{align*}
Note that in the last equality, we use $\frac{\gD(X=x)}{\gD(Z=g(x))} = \frac{\gD(X=x,Z=g(x))}{\gD(Z=g(x))} = \gD(X=x|Z=g(x))$.
\end{proof}

\begin{proof}[Proof of \Thmref{thm:det_cs_di}]
We construct a feature transform to demonstrate that DI assumption is satisfied. Let $f$ denote the deterministic mapping from data to label. Now consider $f$ as the feature transform. Apparently, function $f$ preserve all the label information, i.e. $\gD(Y=y|X=x) = \gD(Y=y|f(X)=f(x))$. Second, the joint distribution of feature and label is the same in the two domains since $\gD_S(Z=f(x),Y=y) = \gD_S(Y=y)\cdot \1[y=f(x)] = \gD_T(Y=y)\cdot \1[y=f(x)] = \gD_T(Z=f(x),Y=y)$. The equality holds since we have the condition that label distributions are matched, i.e., $\gD_S(Y=y)=\gD_T(Y=y)$.
\end{proof}

\begin{proof}[Proof of \Thmref{thm:det_fjs_gls}]
We construct a feature transform to demonstrate that GLS assumption is satisfied. Let $f$ denote the deterministic mapping from data to label. Now consider $f$ as the feature transform. Apparently, function $f$ preserve all the label information, i.e. $\gD_S(Y=y|X=x) = \gD_S(Y=y|f(X)=f(x))$. Second, the feature distribution conditioned on the label is the same in the two domains. $\gD_S(f(X)|Y) = \gD_S(Y|Y) = \gD_T(Y|Y) = \gD_T(f(X)|Y)$.
\end{proof}

\begin{proof}[Proof of Toy Example Not Satisfying GLS] Our toy example has a property that for any two data $x_1 \neq x_2$, $\gD(Y|X=x_1) \neq \gD(Y|X=x_2)$. Thus a feature transformation $g$ preserving all label information has to mapping $x_1,x_2$ to two different features $g(x_1)$ and $g(x_2)$. It means feature transformation is invertible, i.e. $x=g^{-1}(g(x))$. Thus $\gD_S(X|Y) \neq \gD_T(X|Y)$ implies $\gD_S(g(x)|Y) \neq \gD_T(g(x)|Y)$. Hence, it does not satisfy GLS. 
\end{proof}

\begin{lemma}\label{lem:opt}
Let $\sP$ and $\sQ$ are distributions over space $\gX$ which satisfy $\sP$ has a larger support than $\sQ$. Consider the following optimization, $w^*(x) = \arg\min_w \E_{x \sim \sP} \log(1+w(x)) +\E_{x \sim \sQ} \log(1 + \frac{1}{w(x)})$. We claim the optimal solution is the ratio between $\sQ$ and $\sP$, i.e, $w(x) = \frac{\sQ(x)}{\sP(x)}$. The optimal value equals $2(\log 2 - \mbox{JSD}(\sP||\sQ))$ where JSD is Jensen–Shannon divergence, $\mbox{JSD}(\sP||\sQ) = H(\frac{\sP+\sQ}{2}) - \frac{1}{2}(H(\sP)+H(\sQ))$.
\begin{proof}
Considering function $f(w) = p\log(1+w) + q\log(1+w^{-1})$, we have $f'(w) =  \frac{p}{w(1+w)}(w - \frac{q}{p})$. It indicates $f(w)$ reach's its optimal at $w=\frac{q}{p}$. Thus, the original optimization, $\int_{x} ( \sP(x) \log(1+w(x)) + \sQ(x) \log(1 + \frac{1}{w(x)} ) dx$ achieves optima at $w(x)=\frac{\sQ(x)}{\sP(x)}$.
\end{proof}
\end{lemma}

\begin{proof}[Proof of \Proref{prop:sup}]
By the \Lemref{lem:opt}, we know $\gL_{\rm sup}$ achieves minima when $U,V$ satisfy $U(x)V(y)=\frac{\gD_T(x,y)}{\gD_S(x,y)}$. With the FJS assumption, we have $\frac{\gD_T(x,y)}{\gD_S(x,y)}=\gU(x)\gV(y)$.
\end{proof}

\begin{proof}[Proof of \Proref{prop:unsup}] By the \Lemref{lem:opt}, we know $\gL_{\rm unsup}$ achieves minima when $U,\widetilde{V}$ satisfy $U(x)\widetilde{V}(x)=\frac{\gD_T(x)}{\gD_S(x)}$which is equivalent to $U(x)\int_{y} \gD_S(y|x) V(y) dy=\frac{\gD_T(x)}{\gD_S(x)}$.
\end{proof}


\section{More Methodology Details}
\label{app:baseline}

In this section, we detail the baseline methods introduced in Section \ref{sec:exp}, where one representative method is introduced for each traditional domain adaptation assumption.
\begin{itemize}[leftmargin=*]
    \item \emph{Source-Only} method simply trains the model in the source domain and directly apply it to the target domain. The method assumes there exists no shift between source and target domains.
    \item \emph{Sample Select Bias Correction~(SSBC)}~\cite{SSBC} employs a domain classifier to predict the odds for a certain data point $x$ of being from two domains, i.e., $\gD_S(x)/\gD_T(x)$. It then uses the reciprocal odds to re-weight the source domain data during training. Intuitively, SSBC follows the covarite shift assumption.
    \item \emph{Black Box Shift Correction~(BBSC)}~\cite{BBSC} first trains a predictor in the source domain. Then it evaluates the predictor's confusion matrix as well as its predicted label distribution in the target domain. The label shift importance $\gD_T(y)/\gD_S(y)$ can be estimated based on the results. Finally, the estimated importance weights are leveraged to re-weight the training data from the source domain. We note that BBSC follows the label shift assumption.
    \item \emph{Domain Adversarial Neural Networks~(DANN)}~\cite{DANN} adopts adversarial training the to align the feature distributions between source and target domains. A predictor is trained based on such domain invariant features. The underlying assumption of DANN is the domain invariance.
    \item \emph{Importance Weighting Domain Adversarial Networks (IWDAN)} \cite{GLS} goes behind the Generalized Label Shift assumption, where importance weights are estimated across each category between source and target domains. The estimated weights are further combined with adversarial training, and finally formulate an importance weighted domain adversarial training objective.
\end{itemize}

%
%
%

\end{document}